\documentclass{article}

\usepackage{arxiv}

\usepackage[utf8]{inputenc} 
\usepackage[T1]{fontenc}    
\usepackage{hyperref}       
\usepackage{url}            
\usepackage{booktabs}       
\usepackage{amsfonts}       
\usepackage{nicefrac}       
\usepackage{microtype}      
\usepackage{tikz}
\usepackage{amsmath,amssymb,amsthm}
\usepackage{subfig}

\usepackage{algorithm}
\usepackage{algorithmic}

\newtheorem{example}{Example}

\newtheorem{definition}{Definition}
\newtheorem{lemma}{Lemma}
\newtheorem{proposition}{Proposition}
\newtheorem{observation}{Observation}

\newcommand{\af}{\mathit{AF}}
\newcommand{\straf}{\mathit{StrAF}}
\newcommand{\A}{\mathcal{A}}
\newcommand{\R}{\mathcal{R}}
\newcommand{\s}{\mathcal{S}}

\DeclareMathOperator{\coval}{coval}
\DeclareMathOperator{\cf}{cf}
\DeclareMathOperator{\ad}{ad}
\DeclareMathOperator{\stb}{st}
\DeclareMathOperator{\pr}{pr}
\DeclareMathOperator{\co}{co}

\newcommand{\NP}{\mathsf{NP}}
\newcommand{\coNP}{\mathsf{coNP}}

\title{Admissibility in Strength-based Argumentation: Complexity and Algorithms\\Extended Version with Proofs}

\author{
  Yohann Bacquey \\
  Universit\'e Paris Cit\'e, LIPADE\\
  F-75006 Paris\\
  \texttt{yohann.bacquey@etu.u-paris.fr} \\
   \And
 Jean-Guy Mailly \\
  Universit\'e Paris Cit\'e, LIPADE\\
  F-75006 Paris\\
  \texttt{jean-guy.mailly@u-paris.fr} \\
  \AND
  Pavlos Moraitis \\
  Universit\'e Paris Cit\'e, LIPADE \\
  F-75006 Paris\\
  \texttt{pavlos.moraitis@u-paris.fr} \\
  \And
  Julien Rossit \\
  Universit\'e Paris Cit\'e, LIPADE \\
  F-75006 Paris\\
  \texttt{julien.rossit@u-paris.fr} \\
}

\begin{document}
\maketitle

\begin{abstract}
Recently, Strength-based Argumentation Frameworks (StrAFs) have been proposed to model situations where some quantitative strength is associated with arguments. In this setting, the notion of accrual corresponds to sets of arguments that collectively attack an argument. Some semantics have already been defined, which are sensitive to the existence of accruals that collectively defeat their target, while their individual elements cannot. However, until now, only the surface of this framework and semantics have been studied. Indeed, the existing literature focuses on the adaptation of the stable semantics to StrAFs. In this paper, we push forward the study and investigate the adaptation of admissibility-based semantics. Especially, we show that the strong admissibility defined in the literature does not satisfy a desirable property, namely Dung's fundamental lemma. We therefore propose an alternative definition that induces semantics that behave as expected. We then study computational issues for these new semantics, in particular we show that complexity of reasoning is similar to the complexity of the corresponding decision problems for standard argumentation frameworks in almost all cases. We then propose a translation in pseudo-Boolean constraints for computing (strong and weak) extensions. We conclude with an experimental evaluation of our approach which shows in particular that it scales up well for solving the problem of providing one extension as well as enumerating them all.
\end{abstract}

\section{Introduction}
Among widespread knowledge representation and reasoning techniques proposed in the literature of Artificial Intelligence over the last decades, Abstract Argumentation \cite{Dung95} is an intuitive but yet powerful tool for dealing with conflicting information. 
%
Since then, the initial work of Dung has been actively extended and enriched in many directions, {\em e.g.} considering other kinds of relations between arguments \cite{CayrolL13} or additional information associated with arguments or attacks \cite{AmgoudC02,bcp02}. Among them, Strength-based Argumentation Frameworks (StrAFs) \cite{RossitMDM21} allow to associate a quantitative information with each argument. This information is a weight that intuitively represents the intrinsic strength of an argument, and is then naturally combined with attacks between arguments to induce a defeat relation that allows either to confirm an attack between two arguments or to cancel it, if the attacked argument is stronger than the attacker (w.r.t. their respective weights). StrAFs extend further this notion of defeat among arguments by building a defeat that is based on a collective attack of a group of arguments (or accrual) and by offering associate semantics. Within these semantics, arguments can collectively defeat arguments that they cannot defeat individually. Intuitively speaking, these accrual-sensitive semantics allow some kind of compensation among arguments, where the accumulation of weak arguments can create a synergy and get rid of a stronger one they collectively attack. This reasoning approach allows to produce extensions that are not considered when applying classical semantics.


In \cite{RossitMDM21} the authors presented the basics of StrAFs inspired by Dung's semantics for abstract argumentation along with some theoretical and computational results concerning classical issues related to abstract argumentation ({\em i.e.} acceptability semantics, semantics inclusion, extensions existence and verification, etc.). In this paper we propose a state of the art advancement in StrAFs by presenting original theoretical and computational results related to different aspects. More particularly the contribution of this work lies into the following aspects. The semantics proposed in \cite{RossitMDM21} exist in two versions (namely strong, and weak). Roughly speaking, a set is strongly conflict-free iff none of its elements attacks another one, whereas a set is weakly conflict-free iff it does not contain any (successful) accrual against one of its elements. After detecting that strong admissibility fails to satisfy a desirable property in Dung's abstract argumentation frameworks, namely his Fundamental Lemma \cite{Dung95}, we propose an alternative definition for strong admissibility in order to remedy this issue and we define new admissibility-based semantics for StrAFs. 
Furthermore, we study the complexity of reasoning with these semantics and in particular we show that, surprisingly, the complexity does not increase with respect to the complexity of reasoning with standard AFs. For computing the extensions under these semantics, we propose algorithms based on pseudo-Boolean constraints. 



\section{Background Notions}\label{section:bacground}
We assume that the reader is familiar with abstract argumentation \cite{Dung95}. We consider finite {\em argumentation frameworks} (AFs) $\langle \A,\R\rangle$, where $\A$ is the set of {\em arguments}, and $\R\subseteq \A \times \A$ is the {\em attack relation}. We will use $\cf(\af)$ and $\ad(\af)$ to denote, respectively, the conflict-free and admissible sets of an AF $\af$, and $\co(\af)$, $\pr(\af)$ and $\stb(\af)$ for its extensions under the complete, preferred and stable semantics. For more details on the semantics of AFs, we refer the interested reader to \cite{Dung95,HOFASemantics}.

A {\em Strength-based Argumentation Framework} (StrAF) \cite{RossitMDM21} is a triple $\straf = \langle \A, \R, \s \rangle$ where $\A$ and $\R$ are arguments and attacks, and $\s : \A \rightarrow \mathbb{N}$ is a {\em strength function}. An example of such a StrAF is depicted at Figure~\ref{fig:exemple-straf-background}, where nodes represent arguments, edges represent attacks, and the numbers close to the nodes represent the arguments strength.
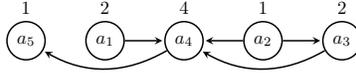
\begin{figure}[ht]
\centering
\scalebox{0.7}{
\begin{tikzpicture}[->,>=stealth,shorten >=1pt,auto,node distance=1.5cm, thick,main node/.style={circle,draw,font=\bfseries},scale=1]
\node[main node,label=above:{$2$}] (A1) {$a_1$};
\node[main node,label=above:{$4$}] (A4) [right of=A1] {$a_4$};
\node[main node,label=above:{$1$}] (A2) [right of=A4] {$a_2$};
\node[main node,label=above:{$2$}] (A3) [right of=A2] {$a_3$};
\node[main node,label=above:{$1$}] (A5) [left of=A1] {$a_5$};

\draw [->] (A1) to (A4);
\draw [->] (A2) to (A4);
\draw [->] (A2) to (A3);
\draw [->] (A3) to[bend left] (A4);
\draw [->] (A4) to[bend left] (A5);

\end{tikzpicture}
} 
\caption{A StrAF Example\label{fig:exemple-straf-background}}
\end{figure}
These strengths intuitively represent the intrinsic robustness associated with an argument and allow to induce a {\em defeat} relation: an argument $a$ defeats another argument $b$ when $a$ attacks $b$ and the strength associated with $b$ does not overcome that with $a$. This framework also offers the notion of collective defeat, {\em i.e.} sets of arguments that can jointly defeat their target while they cannot do so separately. First, we call an {\em accrual} a set of arguments that collectively attack a same target, {\em i.e.} a set $\kappa \subseteq \A$ s.t. $\exists c \in \A$ s.t. $\forall$ $a \in \kappa$, $(a,c) \in \R$. Moreover, we say that $\kappa$ is \emph{an accrual that attacks $c$}. Then, for $\kappa' \subseteq \A$ an accrual, $\kappa$ \emph{attacks} $\kappa'$ iff $\exists a \in \kappa'$ s.t. $\kappa$ attacks $a$.

\begin{example}\label{example:accruals}
Consider again $\straf$ from Figure~\ref{fig:exemple-straf-background}. We can observe several examples of accruals, {\em e.g.} $\kappa_1 = \{a_1, a_2\}$ and $\kappa_2 = \{a_1, a_3\}$, that both attack $a_4$. Notice that any attack $(a_i,a_j) \in \R$ induces an accrual $\{a_i\}$ attacking $a_j$.
\end{example}

We need to assess the collective strength of an accrual.


\begin{definition}[Collective Strength]\label{def:coval-properties}
Let $\straf =\langle \A, \R, \s \rangle$ be a StrAF and $\kappa=\{a_1,...,a_n\} \subseteq \A$ be an accrual. Then the collective strength associated with $\kappa$ is $\coval_\oplus(\kappa) = \oplus(\s(a_1),\dots, \s(a_n))$
where $\oplus$ is an aggregation operator.
\end{definition}

The operator $\oplus$ must satisfy some properties discussed in \cite{RossitMDM21}. An example of suitable operator is $\oplus = \sum$. If $\oplus$ is clear from the context, we simply write $\coval$ for $\coval_\oplus$.

\begin{definition}[Collective Defeat]\label{def:collective-defeat}
Let $\straf =\langle \A, \R, \s \rangle$ be a StrAF, $a \in \A$, and $\oplus$ an aggregation operator. Then, an accrual $\kappa$ \emph{defeats} $a$ with respect to $\coval_\oplus$, denoted by $\kappa \rhd_{\oplus} a$, iff $\kappa \subseteq \A$ is an accrual that attacks $a$ and $\coval_\oplus(\kappa) \geq \s(a)$.
If $\oplus$ is clear from the context, we use $\kappa \rhd a$ instead of $\kappa \rhd_{\oplus} a$.
\end{definition}

In the rest of the paper, we focus on $\oplus = \sum$ in examples and the pseudo-Boolean encoding defined in Section~\ref{section:pb-encoding}. But, unless explicitly stated otherwise, our results remain valid for any $\oplus$ satisfying the properties from \cite{RossitMDM21}.

\begin{definition}\label{def:defeat-of-accrual}
Let $\straf =\langle \A, \R, \s \rangle$ be a StrAF, $\oplus$ an aggregation operator and $\kappa \subseteq \A$, $\kappa' \subseteq \A$ two accruals. Then $\kappa$ \emph{defeats} $\kappa'$, denoted by $\kappa \rhd_\oplus \kappa'$, iff $\exists a \in \kappa'$ s.t. $\kappa \rhd_\oplus a$.
\end{definition}

\begin{example}
Continuing Example~\ref{example:accruals}, notice that $\coval_{\sum}(\kappa_1) = 3 < \s(a_4)$, so $\kappa_1 \not\!\!\!\rhd a_4$. On the contrary, $\coval_{\sum}(\kappa_2) = 4 \geq \s(a_4)$, so $\kappa_2 \rhd a_4$.
\end{example}


StrAF semantics rely on two possible adaptations of the notion of conflict-freeness:

\begin{definition}[Conflict-freeness/Defense] Given $\straf =\langle \A, \R, \s \rangle$ a StrAF, $\oplus$ an aggregation operator, and $S \subseteq \A$,
\begin{itemize}
    \item $S$ is \emph{strongly conflict-free} iff $\nexists a,b \in S$ s.t. $(a,b) \in \R$.
    \item $S$ is \emph{weakly conflict-free} iff there are no accruals $\kappa_1 \subseteq S$ and $\kappa_2 \subseteq S$ s.t. $\kappa_1 \rhd_{\oplus} \kappa_2$.
    \item $S$ defends an element $a \in \A$ iff for all accruals $\kappa_1
      \subseteq \A$, if $\kappa_1 \rhd_{\oplus} a$, then there exists an accrual $\kappa_2
      \subseteq S$ s.t. $\kappa_2 \rhd_{\oplus} \kappa_1$.
\end{itemize}
\end{definition}

Intuitively, strongly conflict-free sets are ``classically" conflict-free, {\em i.e.} there is no attack between two arguments members of such a set. On the contrary, weakly conflict-free sets are ``defeat-free": attacks between arguments are permitted as long as they do not result in a  defeat neither individual nor collective. We use (respectively) $\cf^{\oplus}_S$ and $\cf^{\oplus}_W$ to denote these sets (or simply $\cf_S$ and $\cf_W$ when $\oplus$ is clear from the context).
Then, admissibility and extension-based semantics can be defined either strong or weak. Namely:

\begin{definition}[Semantics for StrAFs \cite{RossitMDM21}]\label{def:acceptability-semantics}
Given $\straf =\langle \A, \R, \s \rangle$ a StrAF, $\oplus$ an aggregation operator, and $S \subseteq \A$ a strong (resp. weak) conflict-free set,
\begin{itemize}
    \item $S$ is a \emph{strong (resp. weak) admissible set} iff $S$ defends all elements of $S$.
    \item $S$ is a \emph{strong (resp. weak) preferred extension} iff
      $S$ is a $\subseteq$-maximal strong (resp. weak) admissible set.
    \item $S$ is a \emph{strong (resp. weak) stable extension} iff $\forall a \in \A \backslash S$, $\exists \kappa \subseteq S$ s.t. $\kappa \rhd_\oplus a$.
\end{itemize}
\end{definition}

For $\sigma$ an extension-based semantics and $X \in \{S,W\}$ meaning respectively {\em strong} and {\em weak}, we use $\sigma^{\oplus}_X(\straf)$ to denote the $X$-$\sigma$ extensions of $\straf$. We drop $\oplus$ from the notation where there is no possible ambiguity.

It is proven in \cite{RossitMDM21} that Dung's AFs are a subclass of StrAFs, where strong and weak semantics coincide. This result is useful for proving complexity results. However, in~\cite{RossitMDM21} authors only focus on complexity issues for the (weak and strong) stable semantics.

\section{Admissibility-based Semantics for StrAFs}\label{section:admissibility}
In our study, we investigate computational issues for admissibility-based semantics. A formal definition of (weak or strong) complete semantics is missing in \cite{RossitMDM21}, but matching the definition of (weak or strong) admissibility with the classical definition of the complete semantics, a straightforward definition can be stated as follows:

\begin{definition}\label{def:complete-grounded}
Let $\straf = \langle \A, \R, \s\rangle$ be a StrAF, and $\oplus$ an aggregation operator. A strong (resp. weak) admissible set $S \subseteq \A$ is a {\em strong (resp. weak) complete extension} of $\straf$ if $S$ contains all the arguments that it defends.
\end{definition}

Now we study these semantics and in particular, we show that surprisingly this intuitive definition of the complete semantics based on strong admissibility fails to satisfy a desirable property, namely the Fundamental Lemma, which states that admissible sets can be extended by the arguments that they defend. This leads us to redefine strong admissibility (and the associated complete and preferred semantics) in Section~\ref{subsection:strong-admissibility}. On the contrary, the definition of weak admissibility is proved to be suitable in Section~\ref{subsection:weak-admissibility}.

\subsection{Revisiting Strong Admissibility}\label{subsection:strong-admissibility}
First, we observe that the usual inclusion relation between the preferred and complete semantics is not satisfied for the strong semantics of StrAFs. Moreover, the universal existence of complete extensions does not hold either.

\begin{proposition}
There exists $\straf$ s.t. $\pr_S(\straf) \nsubseteq \co_S(\straf)$, and $\co_S(\straf) = \emptyset$.
\end{proposition}

\begin{proof}
The strong conflict-free sets of $\straf$ from Figure~\ref{fig:counter-example-pr-co-strong} are $\cf_S(\straf) = \{\emptyset, \{a\}, \{b\},\linebreak \{c\}\}$. Trivially, $\emptyset$ is strongly admissible (since it is not defeated). Similarly, $\{a\}$ is strongly admissible (it is only attacked by $c$, but not defeated). $\{b\}$ is not strongly admissible ($\{a\} \rhd \{b\}$, and $\{b\}$ does not defend itself), neither $\{c\}$ ($\{b\} \rhd \{c\}$, and $\{c\}$ does not defend itself).
With $\ad_S(\straf) = \{\emptyset, \{a\}\}$, $\{a\}$ is the (only) $\subseteq$-maximal strongly admissible set, {\em i.e.} $\pr_S(\straf) = \{\{a\}\}$. However, $\{a\}$ is not complete: $\{a\}$ defends $c$ against all its defeaters, but does not contain it. Hence $\pr_S(\straf) \nsubseteq \co_S(\straf)$. Moreover, $\emptyset$ is not complete either: $\emptyset$ defends $a$ against all its defeaters. Thus $\co_S(\straf) = \emptyset$.
\end{proof}

\begin{figure}[h]
	\centering
	\scalebox{0.7}{
	\begin{tikzpicture}[->,>=stealth,shorten >=1pt,auto,node distance=1.5cm, thick,main node/.style={circle,draw,font=\bfseries},scale=0.5]
	\node[main node,label=above:{$5$}] (a) {$a$};
	\node[main node,label=above:{$4$}] (b) [right of=a] {$b$};
	\node[main node,label=above:{$3$}] (c) [right of=b] {$c$};

	\path[->] (a) edge (b)
			(b) edge (c)
			(c) edge[bend left] (a);
	\end{tikzpicture}
	} 
	\caption{Example proving that $\pr_S(\straf) \nsubseteq \co_S(\straf)$\label{fig:counter-example-pr-co-strong}}
	\end{figure}
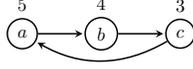

The StrAF from Figure~\ref{fig:counter-example-pr-co-strong} shows that the Fundamental Lemma does not hold for strong semantics of StrAFs: $\{a\}$ is strongly admissible, and defends $c$, but $\{a\} \cup \{c\}$ is not strongly admissible.
A way to solve this issue is to redefine strong admissibility:

\begin{definition}[Strong Semantics Revisited]\label{def:strong-admissibility}\label{def:strong-complete-preferred-revisited}
Let $\straf=\langle \A,\R,  \s\rangle$ be a StrAF and $\oplus$ an aggregation operator. A set $S \in \cf_S(\straf)$ {\em strongly defends} an argument $a$ if $S$ defends $a$ against all the accruals that defeat it, {\em i.e.} $\forall \kappa \subseteq \A$ s.t. $\kappa \rhd a$, $\exists \kappa' \subseteq S$ s.t. $\kappa' \rhd \kappa$, and $S \cup \{a\}$ is strongly conflict-free.
Then, a set $S \subseteq \A$ is {\em strongly admissible} if it is strongly conflict-free and it strongly defends all its elements. Moreover,
\begin{itemize}
    \item $S$ is a \emph{strong preferred extension} iff $S$ is a $\subseteq$-maximal strong admissible set.
    \item $S$ is a \emph{strong complete extension} iff $S$ contains all the arguments that it strongly defends.
\end{itemize}
\end{definition}



If we consider again the StrAF from Figure~\ref{fig:counter-example-pr-co-strong}, observe that this time, the strongly admissible set $\{a\}$ does not strongly defends $c$, since $\{a,c\}$ is not strongly conflict-free. Thus $\{a\}$ is a strong complete extension of this StrAF, following Definition~\ref{def:strong-complete-preferred-revisited}. Now, Dung's Fundamental Lemma can be adapted to strong admissibility.

\begin{lemma}[Fundamental Lemma for Strong Admissibility]\label{lemma:fundamental-strong}
Let $\straf=\langle \A,\R, \s\rangle$ be a StrAF and $\oplus$ an aggregation operator. Let $S \subseteq \A$ be a strongly admissible set, and $a, a'$ two arguments that are strongly defended by $S$ against all their defeaters. Then, $S' = S \cup \{a\}$ is strongly admissible.
\end{lemma}

 \begin{proof}
 The proof follows the definitions of strong defense and strong admissibility: since $S$ strongly defends all its elements, and strongly defends $a$, then $S' = S \cup \{a\}$ is strongly conflict-free and strongly defends all its elements, hence it is strongly admissible. 
 \end{proof}


Lemma~\ref{lemma:fundamental-strong} implies a relation between strong preferred and complete extensions:

\begin{proposition}\label{prop:inclusion-pr-co-strong}
For any $\straf$ and $\oplus$, $\pr^\oplus_S(\straf) \subseteq \co^\oplus_S(\straf)$.
\end{proposition}

 \begin{proof}
 Reasoning with a proof by contradiction, suppose that there is a strong preferred extension $S$ of $\straf$, which is not a strong complete extension. Since $S$ is strongly admissible, it means that it strongly defends an argument $a \in \A \setminus S$. According to Lemma~\ref{lemma:fundamental-strong}, $S \cup \{a\}$ is strongly admissible. Thus we have a strongly admissible set $S' \supset S$, which contradicts the fact that $S$ is a strong preferred extension ({\em i.e.} a $\subseteq$-maximal strong admissible set). This concludes the proof that $S$ is a strong complete extension.
 \end{proof}

This guarantees the existence of at least one strong complete extension for any StrAF: since $\emptyset$ is a strong admissible set for any $\straf$, then $\straf$ admits some $\subseteq$-maximal strong admissible sets, {\em i.e.} $\pr_S(\straf) \neq \emptyset$, which implies $\co_S(\straf) \neq \emptyset$.

\begin{example}
Let us consider again the StrAF provided by Figure~\ref{fig:exemple-straf-background}. 
Its strongly admissible sets are $\ad_S(\straf) = \{\emptyset, \{a_1\}, \{a_2\}, \{a_3\}, \{a_1,a_2\}, \{a_1, a_3, a_5\}\}$. Then, the strong preferred and complete extensions are $\pr_S(\straf) = \co_S(\straf) = \{\{a_1,a_2\}, \{a_1, a_3,a_5\}\}$.
\end{example}

Finally, we prove that the new definition of strong admissibility does not change the fact that strong stable extensions are strongly admissible (and even strong preferred).

\begin{proposition}\label{prop:strong-stable-preferred}
For any $\straf$ and $\oplus$, $\stb^\oplus_S(\straf) \subseteq \pr^\oplus_S(\straf)$.
\end{proposition}

 \begin{proof}
 Let $S \in \stb_S(\straf)$ be a strong stable extension. By definition, $S$ is strongly conflict-free. Let us prove that it is strongly admissible, {\em i.e.} it strongly defends all its elements. Since it is strongly conflict-free, we only need to prove that it (classically) defends all its elements. Indeed, given $a \in S$, the fact that $S \cup \{a\}$ is strongly conflict-free is obvious. Let $\kappa$ be an accrual that defeats $a$. The strong conflict-freeness of $S$ implies that there is no $b \in \kappa$ belonging to $S$, {\em i.e.} $\kappa \subseteq \A \setminus S$. Then, since $S$ defeats all the arguments in $\A \setminus S$ (because it is a strongly stable extension), $S$ defeats all the arguments in $\kappa$. So $a$ is strongly defended. Which proves that $S$ is admissible.

 Now, prove that it is a strongly preferred extension. Assume that it is not, {\em i.e.} $\exists S' \in \ad_S(\straf)$ s.t. $S \subset S'$. This implies the existence of an argument $b \in S' \setminus S$. Since $S$ is strongly stable, there is some accrual $\kappa \subseteq S$ that defeats $b$. This implies that $S'$ is not strongly conflict-free, which contradicts the assumption that $S'$ is a strongly admissible set. So we conclude that no such $S'$ exists, {\em i.e.} $S$ is strongly preferred.
 \end{proof}

\subsection{Properties of the Weak Semantics}\label{subsection:weak-admissibility}
Regarding now weak semantics as defined in \cite{RossitMDM21}, the usual result still holds for StrAFs.

\begin{lemma}[Fundamental Lemma for Weak Admissibility]\label{lemma:fundamental-weak}
Let $\straf=\langle \A,\R, \s\rangle$ be a StrAF and $\oplus$ an aggregation operator. Given $S \subseteq \A$ a weakly admissible set, and $a, a'$ two arguments that are defended by $S$,
\begin{enumerate}
	\item $S' = S \cup \{a\}$ is weakly admissible,
	\item $S'$ defends $a'$.
\end{enumerate}
\end{lemma}

 \begin{proof}
 \begin{enumerate}
 	\item Suppose that $S$ is weakly admissible and $S$ defends $a$. We need to prove that $S \cup \{a\}$ is weakly conflict-free. Reasoning with a proof by contradiction, suppose it is not the case. We split the reasoning in two parts: $\exists \kappa \subseteq S \cup \{a\}$ s.t. $\kappa \rhd a$, or $\kappa \rhd a'$ with $a' \in E$.
 	\begin{enumerate}
 	\item First case: $\exists \kappa \subseteq S \cup \{a\}$ s.t. $\kappa \rhd a$. It means that there is an accrual $\kappa \subseteq S$ s.t. $\kappa \rhd a$. By hypothesis $S$ defends $a$ against $\kappa$, it means that $\exists \kappa' \subseteq S$ s.t. $\kappa' \rhd \kappa$, thus $S$ is not weakly conflict-free. Contradiction.
 	\item Second case: $\exists \kappa \subseteq S \cup \{a\}$ s.t. $\kappa \rhd a'$ with $a' \in S$. By hypothesis, $S$ defends itself against all its defeaters, so $\exists \kappa' \subseteq S$ s.t. $\kappa' \rhd \kappa$. Since $\kappa \subseteq S \cup \{a\}$, either $\kappa'$ defeats $a$ (thus case (a) applies, and this is a contradiction), or $\kappa'$ defeats some argument from $S$, which is also a contradiction.
 	\end{enumerate}
 	\item Obvious from the fact that $S$ defends $a'$.
 \end{enumerate}
 \end{proof}

\begin{proposition}\label{prop:inclusion-pr-co-weak}
For any $\straf$ and $\oplus$, $\pr^\oplus_W(\straf) \subseteq \co^\oplus_W(\straf)$.
\end{proposition}


Similarly to what we have noticed previously for strong admissibility, $\emptyset$ is weakly admissible for any StrAF. This implies the existence of at least one weak preferred extension, and then one weak complete extension for any StrAF.

\begin{example}
Consider again the StrAF from Figure~\ref{fig:exemple-straf-background}. 
One identifies the weakly admissible sets $\ad_W(\straf) = \ad_S(\straf) \cup \{\{a_1,a_3,a_5\}, \{a_2, a_3\}, \{a_1, a_2, a_3\}, \{a_1,a_2,a_3,a_5\}\}\}$. Then, $\pr_W(\straf) =  \co_W(\straf) = \{\{a_1, a_2, a_3, a_5\}\}$.
\end{example}

Notice that, contrary to the case of stable semantics \cite{RossitMDM21}, we do not have $\co_S(\straf) \subseteq \co_W(\straf)$. This comes from the fact that our strong and weak complete semantics are not based on the same notion of defense. However, we observe that each strong complete extension is included in some weak preferred (and complete) extension. We prove that this is true for any StrAF.

\begin{proposition}[Strong/Weak Semantics Relationship]
Given $\straf = \langle \A, \R, \s \rangle $ and $\oplus$ an aggregation operator, $\forall E \in \co_S(\straf)$, $\exists E' \in \pr_W(\straf)$ s.t. $E \subseteq E'$.
\end{proposition}

 \begin{proof}
 The results holds because each strongly admissible set is also a weakly admissible set. Indeed, $E \in \ad_S(\straf)$ is (by definition) strongly conflict-free, which implies its weak conflict-freeness. It also strongly defends all its elements, which implies that it (classically) defends all its elements. Thus $E \in \ad_W(\straf)$. Then, for any $E \in \co_S(\straf)$, by definition $E \in \ad_S(\straf)$, thus $E \in \ad_W(\straf)$. Since weak preferred extension are $\subseteq$-maximal weakly admissible set, the existence of $E' \in \pr_W(\straf)$ s.t. $E \subseteq E'$ holds.
 \end{proof}

Notice finally that we do not need a counterpart to Proposition~\ref{prop:strong-stable-preferred}: the definition of semantics based on weak admissibility is not modified, so the result from \cite[Proposition 1]{RossitMDM21} still holds in this case.

\subsection{Dung Compatibility}
Previous work on StrAFs showed that this framework generalizes Dung's AF, with a correspondence of StrAF semantics with AF semantics in this case.
Following the new definition of strong admissible sets, one might fear that this property does not hold for strong admissibility-based semantics. However, we show here that it still does, as well as for weak complete semantics. Let us recall the transformation of an AF into a StrAF \cite{RossitMDM21}.

\begin{definition}\label{def:associated-strafs}
Given an argumentation framework $\af = \langle \A, \R\rangle$, the StrAF associated with $\af$ is $\straf_{\af} =\langle \A, \R, \s \rangle$ with $\s$ $(a) = 1, \forall a \in \A$ and $\coval = \sum$.
\end{definition}

\begin{observation}
Since all the arguments have the same strength, $(a,b) \in \R$ implies $\{a\} \rhd b$, thus strong and weak conflict-freeness coincide in $\straf_{\af}$.
\end{observation}

We also recall useful lemmas from \cite{RossitMDM21}.

\begin{lemma}~\label{lemma:conflictfree-equivalence}
Let $AF = \langle \A, \R\rangle$ be an AF, and $\mathit{StrAF}_{AF} =\langle \A, \R, \s \rangle$ its associated StrAF. The set $S \subseteq \A$ is conflict-free in $AF$ iff it is strongly conflict-free in $\mathit{StrAF}_{AF}$.
\end{lemma}

Lemma~\ref{lemma:conflictfree-equivalence} is obvious from the definition of strong conflict-freeness. Then Lemma~\ref{lemma:defence-equivalence} is useful for proving Dung Compatibility in the context of weak semantics.

\begin{lemma}\label{lemma:defence-equivalence}
Let $AF = \langle \A, \R\rangle$ be an AF, and $\mathit{StrAF}_{AF} =\langle \A, \R, \s \rangle$ its associated StrAF. The set $S \subseteq \A$ defends the argument $a \in \A$ in $AF$ iff $S \subseteq \A$ defends the argument $a \in \A$ in $\mathit{StrAF}_{AF}$.
\end{lemma}

See \cite{RossitMDM21} for the proof of this Lemma. We can also state a stronger version, that will be useful  for proving that Dung Compatibility holds with the revisited definition of strong admissibility.

\begin{lemma}\label{lemma:strong-defence-equivalence}
Let $AF = \langle \A, \R\rangle$ be an AF, and $\mathit{StrAF}_{AF} =\langle \A, \R, \s \rangle$ its associated StrAF. The set $S \subseteq \A$ defends the argument $a \in \A$ in $AF$ iff $S \subseteq \A$ strongly defends the argument $a \in \A$ in $\mathit{StrAF}_{AF}$.
\end{lemma}

\begin{proof}
Let us suppose that $S \subseteq \A$ defends the argument $a \in \A$ in $AF$. This means that for each $b \in \A$ such that $(b,a) \in \R$, $\exists c \in S$ such that $(c,b) \in \R$. Moreover, from the Fundamental Lemma by \cite{Dung95}, $S \cup \{a\}$ is admissible, hence it is conflict-free. This implies the strong conflict-freeness of $S \cup \{a\}$.
Now, let us consider an accrual $\kappa_1 \subseteq \A$ such that $\kappa_1 \rhd a$. As established previously, $\forall b \in \kappa_1$, $\exists c \in S$ such that $(c,b) \in \R$, {\em i.e.} $\exists \kappa_2 = \{c\} \subseteq S$ such that $\kappa_2$ attacks $b$. Since $\coval(\kappa_2) = \s(c) = 1 = \s(b)$, $\kappa_2 \rhd b$  and thus $\kappa_2 \rhd \kappa_1$. So $S$ strongly defends the argument $a$ in $\mathit{StrAF}_{AF}$.

Now we suppose that $S \subseteq \A$ strongly defends the argument $a \in \A$ in $\mathit{StrAF}_{AF}$, {\em i.e.} $S \cup \{a\}$ is strongly conflict-free, and for all accruals $\kappa_1$ that defeat $a$, $\exists \kappa_2 \subseteq S$ such that $\kappa_2 \rhd \kappa_1$. Since all arguments strengths are equal to $1$, every argument $b \in \A$ attacking $a$ corresponds to an accrual $\kappa_1 = \{b\}$ defeating $a$. So $\exists \kappa_2 \subseteq S$ such that $\kappa_2 \rhd \kappa_1 = \{b\}$, and thus $\exists c \in \kappa_2 \subseteq S$ such that $(c,b) \in \R$. So we conclude that $S$ defends the argument $a$ in $AF$.
\end{proof}

Now we can state the following proposition, that extends Dung Compatibility from \cite{RossitMDM21} to the semantics studied in this paper.

\begin{proposition}[Dung Compatibility] \label{prop:dung-compatibility2} Let $\af = \langle \A, \R\rangle$ be an AF, and $\straf_{\af} = \langle \A, \R, \s \rangle$ from Def.~\ref{def:associated-strafs}. For $\sigma \in \{\ad,\pr,\co\}$, $\sigma(\af) = \sigma_X(\straf_{\af})$, for $X \in \{S,W\}$.
\end{proposition}

\begin{proof}
From Lemmas~\ref{lemma:conflictfree-equivalence} and~\ref{lemma:strong-defence-equivalence}, $S$ is admissible in $\af$ iff $S$ is strongly admissible in $\straf_{\af}$. So the $\subseteq$-maximal admissible sets in $\af$ coincide with the $\subseteq$-maximal strong admissible sets in $\straf_{\af}$, hence Dung Compatibility for strong preferred semantics. For strong complete semantics, notice that an admissible set $S$ defends (in $\af$) an argument $a \in \A \setminus S$ iff $S$ strongly defends $a$ in $\straf_{\af}$, thus $S$ is a complete extension of $\af$ iff it is a strong complete extension of $\straf_{\af}$, and similarly for weak complete semantics using Lemma~\ref{lemma:defence-equivalence}. Results for weak admissibility and weak preferred semantics come from \cite{RossitMDM21}.
\end{proof}

\section{Complexity and Algorithms}\label{section:complexity}
Now we provide some insight on computational issues for admissibility-based semantics of StrAFs, {\em i.e.} we identify the computational complexity of several classical reasoning problems under these semantics, and we provide algorithms (based on pseudo-Boolean encoding) for solving them. While the complexity results are generic regarding the choice of $\oplus$, the algorithms focus on $\oplus = \sum$.

\subsection{Complexity Analysis}
We assume that the reader is familiar with basic notions of complexity, and otherwise we refer the interested reader to \cite{HOFAComplexity} for details on complexity in formal argumentation, and \cite{AroraB2009} for a more general overview of computational complexity.
We focus on three classical reasoning problems in abstract argumentation, namely {\em verification} (``Is a given set of arguments an extension?"), {\em credulous acceptability} (``Is a given argument member of some extension?") and {\em skeptical acceptability} (``Is a given argument member of each extension?"). Formally, for $\sigma \in \{\ad,\pr,\co\}$ and $X \in \{S,W\}$:
\begin{itemize}
    \item $\sigma$-$X$-Ver: Given $\straf=\langle \A,\R, \s\rangle$ and $S \subseteq \A$, is $S$ a member of $\sigma_X(\straf)$?
    \item $\sigma$-$X$-Cred: Given $\straf=\langle \A,\R, \s\rangle$ and $a \in \A$, is $a$ in some $S \in \sigma_X(\straf)$?
    \item $\sigma$-$X$-Skep: Given $\straf=\langle \A,\R, \s\rangle$ and $a \in \A$, is $a$ in each $S \in \sigma_X(\straf)$?
\end{itemize}

We recall that these reasoning problems are already considered only for the (weak and strong) stable semantics in \cite{RossitMDM21}.
In the following, we assume a fixed $\oplus$, that can be computed in polynomial time. This is not a very strong assumption, since it is the case with the classical aggregation operators ({\em e.g.} $\sum, \max,\dots$).
Proposition~\ref{prop:dung-compatibility2} implies that the complexity of reasoning with standard AFs provides a lower bound complexity of reasoning with StrAFs. So we focus on identifying upper bounds.

\begin{proposition}[Verification]\label{proposition:complexity-verification}
For $X \in \{S,W\}$, $\sigma$-$X$-Ver $\in P$, for $\sigma \in \{\ad,\co\}$, and $\pr$-$X$-Ver is $\coNP$-complete.
\end{proposition}

 \begin{proof}
 Let us start with strong (resp. weak) admissible sets. Given $S \subseteq \A$, to verify whether $S$ is a strong (resp. weak) admissible set, one must:
 \begin{itemize}
     \item check whether it is a strong (resp. weak) conflict-free set: doable in polynomial time (see \cite{RossitMDM21}),
     \item for each $\kappa$ s.t. $\kappa \rhd S$ (that can be identified in polynomial time by checking the attackers of $S$), check whether $\exists \kappa' \subseteq S$ s.t. $\kappa' \rhd \kappa$ (doable in polynomial time by checking the attackers of $\kappa$).
 \end{itemize}
 Hence $\ad$-$X$-Ver $\in P$, for $X \in \{S,W\}$.\\ 
 \\
 Now we focus on strong (resp. weak) complete semantics. One first needs to check (polynomially) whether $S$ is a strong (resp. weak) admissible set. If yes, then check whether it defends some $a \not\in S$ (doable in polynomial time by checking if there are accruals $\kappa$ defeating $a$, and whether some $\kappa' \subseteq S$ defeats $\kappa$). For strong complete semantics, restrict this part to arguments $a \in \A \setminus S$ s.t. $S \cup \{a\}$ is strongly conflict-free (this can also be verified in polynomial time). Thus $\co$-$X$-Ver $\in P$, for $X \in \{S,W\}$.

 Finally, for checking whether a set $S \subseteq A$ is a strong (resp. weak) preferred extension, first check whether it is strongly (resp. weakly) admissible, then non-deterministically guess a proper superset $S' \supset S$, and check whether $S'$ is strongly (resp. weakly) admissible. A positive answer proves that $S$ is not a strong (resp. weak) preferred extension. Hence $\pr$-$X$-Ver $\in \coNP$, for $X \in \{S,W\}$. Then, Dung compatibility (Proposition~\ref{prop:dung-compatibility2}) and known complexity results about AFs \cite{HOFAComplexity} prove that $\pr$-$X$-Ver is $\coNP$-complete.
 \end{proof}

\begin{proposition}[Credulous Acceptability]\label{proposition:complexity-credoulous}
$\sigma$-$X$-Cred is $\NP$-complete, for $\sigma \in \{\ad,\co,\pr\}$ and $X \in \{S,W\}$.
\end{proposition}

 \begin{proof}
 A classical non-deterministic algorithm can be used for checking the credulous acceptability of an argument $a$: guess a set of arguments $S \subseteq \A$ s.t. $a \in S$, then (polynomially) check whether $S$ is a strong (resp. weak) admissible set. This approach guarantees that $\ad$-$X$-Cred $\in \NP$, for $X \in \{S,W\}$. Since the verification is also polynomial for strong (resp. weak) complete extensions, the reasoning holds for proving that $\co$-$X$-Cred $\in \NP$. Finally, from known complexity results \cite{HOFAComplexity} and Dung compatibility (Proposition~\ref{prop:dung-compatibility2}), we deduce that $\ad$-$X$-Cred and $\co$-$X$-Cred are $\NP$-hard, thus we conclude that these problems are $\NP$-complete.

 Finally, since strong (resp. weak) preferred extensions are $\subseteq$-maximal strong (resp. weak) admissible sets, credulous acceptability under strong (resp. weak) preferred semantics is equivalent to credulous acceptability under strong (resp. weak) admissibility. So $\pr$-$X$-Cred is $\NP$-complete too.
 \end{proof}

\begin{proposition}[Skeptical Acceptability]\label{proposition:complexity-skeptical}
For $X \in \{S,W\}$, $\ad$-$X$-Skep is trivial, $\co$-$X$-Skep $\in \coNP$, and $\pr$-$X$-Skep is $\Pi_2^P$-complete.
\end{proposition}

 \begin{proof}
 Since $\emptyset \in \ad_X(\straf)$, for any $\straf$, with $X \in \{S,W\}$, the answer is trivially ``NO" for any $\ad$-$X$-Skep instance.

 We give an upper bound for the complexity of skeptical acceptance under the strong (resp. weak) complete semantics. A non-deterministic algorithm solves it by guessing a set of arguments $S \subseteq \A$ s.t. $a \not\in S$, and checking (in polynomial time) whether $S$ is a strong (resp. weak) complete extension. So $\co$-$X$-Skep $\in \coNP$.

 Finally, known complexity results \cite{HOFAComplexity} and Dung compatibility (Proposition~\ref{prop:dung-compatibility2}) allow us to deduce that $\pr$-$X$-Skep is $\Pi_2^P$-hard for $X \in \{S,W\}$. Then, this problem is solved by non-deterministically guessing a set of arguments $S \subseteq \A$ s.t. $a \not\in S$, and then checking (with a $\coNP{}$ oracle) that $S$ is not a strong (resp. weak) preferred extension (proving that $a$ is not skeptically accepted). This proves that $\pr$-$X$-Skep is actually $\Pi_2^P$-complete.
 \end{proof}

Proposition~\ref{prop:complexity-all-results} summarizes the results given above.

\begin{proposition}\label{prop:complexity-all-results}
The complexity of the decision problems $\sigma$-$X$-Ver, $\sigma$-$X$-Cred and $\sigma$-$X$-Skep is as described in Table~\ref{tab:complexity-results}.
\end{proposition}


 \begin{table}[ht]
 \centering
 \begin{tabular}{c|c|c|c}
  & $\sigma$-$X$-Ver & $\sigma$-$X$-Cred & $\sigma$-$X$-Skep \\ \hline
  $\ad_X$ & $P$ & $\NP$-c & Trivial \\
  $\co_X$ & $P$ & $\NP$-c & in $\coNP$ \\
  $\pr_X$ & $\coNP$-c & $\NP$-c & $\Pi_2^P$-c \\
 \end{tabular}
 \caption{Complexity of reasoning for $\sigma_X$ with $\sigma \in \{\ad,\co,\pr\}$ and $X \in \{S,W\}$. Trivial means that all instances are trivially ``NO" instances, and $\mathcal{C}$-c means $\mathcal{C}$-complete, for $\mathcal{C}$ a complexity class in the polynomial hierarchy.\label{tab:complexity-results}}
 \end{table}

%
As it is the case for the strong (resp. weak) stable semantics \cite{RossitMDM21}, we prove here that the higher expressivity of StrAFs (compared to AFs) does not come at the price of a complexity blow-up. Only the case of skeptical acceptability under strong (resp. weak) complete semantics requires a deeper analysis, since we only provide the $\coNP$ upper bound. 
We observe that the choice of the weak or strong variant of the semantics does not have an impact on the complexity of reasoning.
%

\subsection{Algorithms}\label{section:pb-encoding}
For computing the strong (resp. weak) admissible sets and complete extensions, we propose a translation of StrAF semantics in pseudo-Boolean (PB) constraints \cite{RousselM09}. Such a constraint is an (in)equality $\sum_i w_i \times l_i \# k$ where $w_i$ and $k$ are positive integers, and $\# \in \{>,\geq,=, \neq,\leq,<\}$. $l_i$ is a literal, {\em i.e.} $l_i = v_i$ or $l_i = \overline{v_i} = 1 - v_i$, where $v_i$ is a Boolean variable. Determining whether a set of PB constraints has a solution is a $\NP$-complete problem, that generalizes the Boolean satisfiability (SAT) problem. Despite the high complexity of this problem, it can be efficiently solved in many cases, see {\em e.g.} \cite{MartinsML14,ElffersN18}. 

\paragraph{Strong and Weak Conflict-freeness}

Now we describe our PB encoding of StrAF semantics. For ensuring self-containment of the paper, we recall the encoding of strong and weak conflict-freeness \cite{RossitMDM21}. Given $\mathit{StrAF} = \langle\A, \R,\s\rangle$ and $\coval = \sum$, we define a set of Boolean variables $\{x_i \mid a_i \in \A\}$ associated with each argument, where $x_i = 1$ means that $a_i$ belongs to the set of arguments characterized by the solutions of the PB constraints. Then, strong conflict-freeness is encoded by:
\begin{description}
    \item[(1)] $\forall (a_i, a_j) \in \R$, add the constraint $x_i + x_j \leq 1$
\end{description}
and weak conflict-freeness is encoded by:
\begin{description}
    \item[(1')] $\forall a \in \A$, add the constraint $\sum_{a_i \in \Gamma^-(a)}\s(a_{i})\times  x_{i} < x\times \s(a) + \overline{x} \times M$
\end{description}
with $M$ an arbitrary large natural number that is greater than the sum of the strengths of the arguments ({\em i.e.} $M > \sum_{a \in \A} \s(a)$), $\Gamma^-(a) = \{b \mid (b,a) \in \R\}$ is the set of attackers of $a \in \A$, and $x$ is the Boolean variable associated with $a$.\footnote{Notice that the constraints referring to $\Gamma^-(a)$ must be added even when $\Gamma^-(a) = \emptyset$.}
A solution to the set of constraints {\bf (1)} (resp. {\bf (1')}) yields a strong (resp. weak) conflict-free set $E = \{a_i \mid x_i = 1\}$. We prove this claim with Proposition~\ref{proposition:encoding-strong-cf}.
First, let us introduce some notations. Given $S \subseteq \A$, $\omega_S : \{x_i \mid a_i \in \A\} \rightarrow \{0,1\}$ is a mapping s.t. $\omega_S(x_i) = 1$ iff $a_i \in S$. 

\begin{proposition}\label{proposition:encoding-strong-cf}\label{proposition:encoding-weak-cf}
Given $\straf = \langle\A, \R,\s\rangle$ and $S \subseteq \A$, $S \in \cf_S(\straf)$ (resp. $S \in \cf_W(\straf)$) iff $\omega_S$ satisfies the set of constraints {\bf (1)} (resp. {\bf (1')}).
\end{proposition}

\begin{proof}
We start with strong conflict-freeness.
Suppose that $S$ is strongly conflict-free. This means that, for any attack $(a_i,a_j) \in \R$, at most one of $a_i$ and $a_j$ belongs to $S$, {\em i.e.} at most one $x_i$ and $x_j$ is equal to one, hence $x_i + x_j \leq 1$. 

On the contrary, suppose that $\omega_S$ satisfies the set of constraints {\bf (1)}. This means that, for any two arguments $a_i, a_j \in S$ ({\em i.e.} both $x_i$ and $x_j$ are equal to $1$), if there is an attack between them, the constraint $x_i + x_j \leq 1$ is falsified. This is a contradiction with the hypothesis that $\omega_S$ satisfies the constraint, thus there cannot be an attack between arguments in $S$, {\em i.e.} $S$ is strongly conflict-free.

Now we focus on weak conflict-freeness.
Suppose that $S$ is weakly conflict-free. This means that there is no $\kappa \subseteq S$ and $a \in S$ s.t. $\kappa \rhd a$. First, if $S$ is actually strongly conflict-free, then all the constraints from the set {\bf (1')} with an argument $a \in S$ in the right-hand part are trivially satisfied (they are reduced to $0 < \s(a)$, because $x_j = 0$ for any attacker $a_j \in \A \setminus S$ of $a$). Now consider the case where $S$ is weakly conflict-free without being strongly conflict-free. This means that there are attacks between arguments in $S$. Consider some $a \in S$ that is attacked by at least one argument $b \in S$. The constraint from the set {\bf (1')} with $a$ in its right-hand side becomes $\sum_{a_j \in S, (a_j,a) \in R} \s(a_j)< \s(a)$. This constraint is satisfied iff the accrual $\kappa = \{a_j \in S \mid (a_j,a) \in R\}$ (and any $\kappa' \subseteq \kappa$) does not defeat $a$. This is the case since $S$ is weakly conflict-free. Finally, consider the constraints from {\bf (1')} with some argument $a \not \in S$ on the right-hand side. The right-hand side is reduced to $M$, which is arbitrarily large, thus the constraint is satisfied.

Now suppose that $\omega_S$ is a solution of the set of constraints {\bf (1')}. For any $a \in S$ without any attack in $S$, the constraint with $a$ on the right-hand side is trivially satisfied (it becomes $0 < \s(a)$). Then, for any $a \in S$ with some attackers in $S$, the constraint with $a$ on the right-hand side becomes $\sum_{a_j \in S, (a_j,a) \in R} \s(a_j) < \s(a)$. This constraint is satisfied iff the accrual $\kappa = \{a_j \in S \mid (a_j,a) \in R\}$ (and any $\kappa' \subseteq \kappa$) does not defeat $a$, hence the conclusion.
\end{proof}

\paragraph{Strong and Weak Admissibility} For encoding strong (resp. weak) admissibility, one must add to the set of constraints {\bf (1)} (resp. {\bf (1')}) some new constraints that represent the strong defense (resp. defense) property. To do so, one needs to introduce new Boolean variables $\{y_i \mid a_i \in \A\}$ s.t. $y_i = 1$ means that $a_i$ is defeated by the set of arguments characterized by the solution of the PB constraints. Then, three constraints are added (the same ones for strong and weak admissibility):
\begin{description}
    \item[(2)] $\forall a \in \A$, add the constraint $\sum_{a_i \in \Gamma^-(a)}\s(a_{i})\times  x_{i} \geq y\times \s(a)$ 
    \item[(3)] $\forall a \in \A$, add the constraint $\sum_{a_i \in \Gamma^-(a)}\s(a_{i})\times  x_{i} \leq \overline{y} \times \s(a) +  y \times M$
    \item[(4)] $\forall a \in \A$, add the constraint $\sum_{a_i \in \Gamma^-(a)}\s(a_i) \times \overline{y_i} \leq x \times \s(a) + \overline{x} \times M$
\end{description}
The sets of constraints {\bf (2)} and {\bf (3)} ensure that $y = 1$ iff $a$ is defeated by some $\kappa \subseteq E = \{a_i \mid x_i = 1\}$, and the constraints {\bf (4)} ensure that $E$ defends all its elements. The following proposition shows the correctness of the encodings.

\begin{proposition}\label{proposition:encoding-strong-admissible}\label{proposition:encoding-weak-admissible}
Given $\straf = \langle\A, \R,\s\rangle$ and $S \subseteq \A$, $S \in \ad_S(\straf)$ (resp. $S \in \ad_W(\straf)$) iff $\omega_S$ satisfies the sets of constraints {\bf (1)} (resp. {\bf (1')}), {\bf (2)}, {\bf (3)} and {\bf (4)}.
\end{proposition}

\begin{proof}
We start with strong admissibility.
Suppose that $S$ is a strongly admissible set. Strong conflict-freeness implies that $\omega_S$ satisfies the set of constraints {\bf (1)} (see Proposition~\ref{proposition:encoding-strong-cf}).

Now, let $a$ be an argument defeated by some accrual $\kappa \subseteq S$. The constraint from set {\bf (2)} with $a$ on the right-hand side becomes $\sum_{a_j \in S, (a_j,a) \in \R} \s(a_j) \geq y \times \s(a)$. The constraint is satisfied, since the sum of the strengths of the arguments in $\kappa$ is greater than the strength of $a$, the value of $y$ does not matter. The constraint from set {\bf (3)} with $a$ on the right-hand side becomes $\sum_{a_j \in S, (a_j,a) \in \R} \s(a_j) \leq \overline{y} \times \s(a) +  y \times M$. Since the collective strength of the attackers of $a$ in $\kappa$ is greater than the strength of $a$, the constraint is satisfied when $y = 1$ (recall that $M$ is an arbitrary large integer).

Now consider an argument $a$ that is not defeated by any accrual $\kappa \subseteq S$. This means that the sum of the strengths of the arguments in $S$ that attack $a$ is lesser than the strength of $a$, thus the constraint $\sum_{a_j \in S, (a_j,a) \in \R} \s(a_j) \geq y \times \s(a)$ is satisfied iff $y = 0$. Then, the constraint $\sum_{a_j \in S, (a_j,a) \in \R} \s(a_j) \leq \overline{y} \times \s(a) +  y \times M$ is satisfied for any value of $y$ (since the collective strength of the attacks of $a$ in $S$ is lesser than the strength of $a$, and lesser than $M$).

Observe that $\omega_S$ satisfies both the sets of constraints {\bf (2)} and {\bf (3)}, and implies that $y = 1$ iff the associated argument $a$ is defeated by some accrual $\kappa \subseteq S$.

Now we focus on the set of constraints {\bf (4)}. For any argument $a \in \A \setminus S$, the right-hand side of the constraint is reduced to $M$, which is (by definition) greater than the left-hand side. Thus the constraint is satisfied. Now consider some argument $a \in S$. The constraint becomes $\s(a_1) \times \overline{y_1} + \s(a_2) \times \overline{y_2}  + \dots + \s(a_n) \times \overline{y_n} \leq \s(a)$. Recall that $a$ is defended against all the accruals $\kappa$ s.t. $\kappa \rhd a$ (because of the strong admissibility of $S$). For all the attackers $a_j$ that are defeated by $S$, $y_j = 0$. Let $Att_a = \{a_j \in \Gamma^-(a) \mid S \not\rhd a_j\}$ be the set of attackers of $a$ that are not defeated by $S$. The constraint can be re-written $\sum_{a_j \in Att_a} \s(a_j) \leq \s(a)$. The constraint is satisfied, because otherwise $Att_a$ would be an accrual that defeats $a$ and that is not defeated by $S$, which is impossible because of the strong admissibility of $S$.

Now we prove the opposite direction, {\em i.e.} we suppose that $\omega_S$ satisfies the sets of constraints {\bf (1)}, {\bf (2)}, {\bf (3)} and {\bf (4)}. The satisfaction of the constraints {\bf (1)} implies that $S$ is strongly conflict-free (see Proposition~\ref{proposition:encoding-strong-cf}). We must prove that $S$ strongly defends all its elements. Let $a \in S$ be an argument. Strong conflict-freeness of $S$ implies that $S \cup \{a\}$ is strongly conflict-free, thus $S$ strongly defends $a$ iff $S$ ``classically" defends $a$, {\em i.e.} $\forall \kappa \subseteq \A$ s.t. $\kappa \rhd a$, $\exists \kappa' \subseteq S$ that defeats $\kappa$.

Consider an argument $a$ s.t. $\omega_S(y) = 0$. Then the constraints from the sets {\bf (2)} and {\bf (3)}, with $a$ on the right-hand side, become (respectively) $\sum_{a_j \in S, (a_j,a) \in \R} \s(a_j) \geq 0$ (which is trivially satisfied) and $\sum_{a_j \in S, (a_j,a) \in \R} \s(a_j) \leq \s(a)$. This last constraint implies that $a$ is not defeated by $\kappa = \{a_j \in S \mid (a_j,a) \in \R\}$ nor any $\kappa' \subseteq \kappa$, {\em i.e.} there is no $\kappa \subseteq S$ s.t. $\kappa \rhd a$.
On the contrary, consider $a$ s.t. $\omega_S(y) = 1$. The constraint from {\bf (2)} with $a$ in the right-hand side becomes $\sum_{a_j \in S, (a_j,a) \in \R} \s(a_j) \geq \s(a)$. This means that there is an accrual $\kappa = \{a_j \in S \mid (a_j,a) \in \R\} \subseteq S$ s.t. $\kappa \rhd a$. So, for any $a \in \A$, $\omega_S(y) = 1$ iff $a$ is defeated by some $\kappa \subseteq S$.
Now look at the constraints from the set {\bf (4)}. For any $a \in \A \setminus S$, $\omega_S(x) = 0$, so the constraint with $a$ on the right-hand side becomes $\s(a_1) \times \overline{y_1} + \s(a_2) \times \overline{y_2}  + \dots + \s(a_n) \times \overline{y_n} \leq M$, which is trivially satisfied. Now for $a \in S$, $\omega_S(x) = 1$, thus the constraint becomes $\s(a_1) \times \overline{y_1} + \s(a_2) \times \overline{y_2}  + \dots + \s(a_n) \times \overline{y_n} \leq \s(a)$. The left-hand side can be reduced to the sum of the strengths of the attackers of $a$ that are not defeated. Since the constraint implies that this strength is lesser than the strength of $a$, there is no accrual $\kappa$ that defeats $a$ and that is not in turn defeated by some $\kappa' \subseteq S$. So we can conclude that $S$ strongly defends all its elements, and thus it is strongly admissible.

The proof is analogous for weak admissibility.
\end{proof}

\paragraph{Strong and Weak Complete Semantics}
Now, for computing the strong (resp. weak) extensions, one must consider the sets of constraints {\bf (1)} (resp. {\bf (1')}), {\bf (2)}, {\bf (3)} and {\bf (4)}, and add a last set of constraints, respectively {\bf (5)} for strong complete semantics, and {\bf (5')} for weak complete semantics:
\begin{description}
    \item[(5)] $\forall a \in \A$, add the constraint $\sum_{a_i \in \Gamma^-(a)} (\s(a_{i})\times  \overline{y_i}) + \sum_{a_i \in \Gamma^-(a)} (M \times  x_i) + \sum_{a'_i \in \Gamma^+(a)} (M \times  x'_i) \geq \overline{x} \times \s(a)$
    \item[(5')] $\forall a \in \A$, add the constraint $\sum_{a_i \in \Gamma^-(a)}\s(a_{i})\times  \overline{y_i} \geq \overline{x} \times \s(a)$
\end{description}
where $\Gamma^+(a) = \{b \in \A \mid (a,b) \in \R \}$ is the set of arguments attacked by $a$.
These constraints ensure that an argument is not accepted only if it is not (strongly) defended. Again, we prove the correctness of the encodings:

\begin{proposition}\label{proposition:encoding-strong-complete}\label{proposition:encoding-weak-complete}
Given $\straf = \langle\A, \R,\s\rangle$ and $S \subseteq \A$, $S \in \co_S(\straf)$ (resp. $S \in \co_W(\straf)$) iff $\omega_S$ satisfies the sets of constraints {\bf (1)} (resp. {\bf (1')}), {\bf (2)}, {\bf (3)}, {\bf (4)} and {\bf (5)} (resp. {\bf (5')}).
\end{proposition}

\begin{proof}
Let $S \subseteq \A$ be a strong complete extension of $\straf$. Strong admissibility and Proposition~\ref{proposition:encoding-strong-admissible} imply that $\omega_S$ satisfies the set of constraints {\bf (1)}, {\bf (2)}, {\bf (3)} and {\bf (4)}. Let us focus on the set of constraints {\bf (5)}. For any $a \in S$, $\omega_S(x) = 1$, and the constraint is trivially satisfied (the right-hand side becomes $0$). Now consider $a \in \A \setminus S$. Let us define $Att_a = \{a_j \in \A \mid (a_j,a) \in \R, \omega_S(y_j) = 0\}$ the set of attackers of $a$ that are not defeated by any $\kappa \subseteq S$, and $Tar_a = \{a_k \in \A \mid (a,a_k) \in \R, a_k \in S\}$. The constraint becomes $\sum_{a_j \in Att_a} \s(a_j) + |Tar_a| \times M \geq \s(a)$. Suppose first that $Tar_a = \emptyset$, {\em i.e.} there is no argument in $S$ attacked by $a$. The constraint is then $\sum_{a_j \in Att_a} \s(a_j) \geq \s(a)$. Since $S$ is a strong complete extension, it does not strongly defend $a$, {\em i.e.} there is an accrual $\kappa \subseteq Att_a$ s.t. $\kappa \rhd a$, and $\nexists \kappa' \subseteq S$ with $\kappa' \rhd \kappa$. This implies that the collective strength of the arguments in $\kappa$ is greater than the strength of $a$, which means that the constraint is satisfied. No, if $Tar_a \neq \emptyset$, the constraint is satisfied as well because of $|Tar_a| \times M$ on the left-hand side.

Now, for the other direction of the proof, let us suppose that $\omega_S$ satisfies the set of constraints {\bf (1)}, {\bf (2)}, {\bf (3)}, {\bf (4)} and {\bf (5)}. The satisfaction of the sets {\bf (1)}, {\bf (2)}, {\bf (3)} and {\bf (4)} implies that $S$ is strongly admissible (see Proposition~\ref{proposition:encoding-strong-admissible}). Let us show that $S$ is a strong complete extension, {\em i.e.} it does not strongly defend any $a \in \A \setminus S$. Reasoning towards a contradiction, suppose that there is $a \in \A \setminus S$ that is strongly defended by $S$. $a \in \A \setminus S$ implies that $\omega_S(x) = 0$, so the constraint from the set {\bf (5)} becomes $\s(a_{1})\times  \overline{y_1} + \s(a_{2})\times \overline{y_2} + \dots + \s(a_{n})\times \overline{y_n} + M \times x'_1 + \dots + M \times x'_m \geq \s(a)$. Let $Att_a = \{a_j \in \A \mid (a_j,a) \in \R, \omega_S(y_j) = 0\}$ be the set of attackers of $a$ that are not defeated by any $\kappa \subseteq S$ and $Tar_a = \{a_k \in \A \mid (a,a_k) \in \R, a_k \in S\}$ the set of arguments in $S$ that are attacked by $a$. The constraint can then be rewritten $\sum_{a_j \in Att_a} \s(a_j) + |Tar_A| \times M \geq \s(a)$. Since the constraint is satisfied, it means that
\begin{itemize}
    \item either there is an accrual $\kappa = Att_a$ s.t. $\kappa \rhd a$, and there is no $\kappa' \subseteq S$ with $\kappa' \rhd \kappa$;
    \item or there is an argument $a_k \in S$ such that $(a,a_k) \in \R$, {\em i.e.} $S \cup \{a\}$ is not strongly conflict-free.
\end{itemize}
In both cases, there is a contradiction with the assumption that $S$ strongly defends $a$. So there is no such $a$: $S$ is a strong complete extension.

The proof for weak complete semantics is analogous.
\end{proof}

\paragraph{Acceptability and Verification} Obtaining one (resp. each) solution for one of the sets of constraints defined previously corresponds to obtaining one (resp. each) extension of the StrAF under the corresponding semantics.
For checking whether a given argument $a_i$ is credulously accepted, one simply needs to add the constraint $x_i = 1$. If a solution exists, then it corresponds to an extension that contains $a_i$, proving that this argument is credulously accepted. Otherwise, $a_i$ is not credulously accepted. For skeptical acceptability, one needs to add the constraint $x_i = 0$. In this case, a solution exhibits an extension that does not contain $a_i$, thus this argument is not skeptically accepted. In the case where no solution exists, then the argument is skeptically accepted. Finally, for checking whether a set of arguments $S \subseteq \A$ is an extension, one needs to add the constraints $x_i = 1$ for each $a_i \in S$, as well as $x_i = 0$ for each $a_i \in \A \setminus S$. A solution exists for the new set of constraints iff $S$ is an extension under the considered semantics.

\paragraph{Strong and Weak Preferred Semantics}
Finally, let us mention an approach to handle reasoning with strong and weak preferred semantics. Because of the higher complexity of skeptical reasoning under these semantics (recall Proposition~\ref{prop:complexity-all-results}), it is impossible (under the usual assumption that the polynomial hierarchy does not collapse) to find a (polynomial) encoding of these semantics in PB constraints. However, PB solvers can be used as oracles to find (with successive calls) preferred extensions. Algorithm~\ref{computePrefAlgorithm} describes our method to do this for strong preferred semantics (replacing {\bf (1)} by {\bf (1')} provides an algorithm for weak preferred semantics).
At start, we add the four constraints corresponding to a strong (resp. weak) admissible set and solve the instance, with the PB solver as a $\coNP$ oracle. Then we force the arguments within the extension to stay in the next one by adding the constraint on line $4$. To avoid getting the same solution as in the previous step, we make sure that at least one argument outside the previous extension will be in the next one (line $5$). This method iteratively extends an admissible set into a preferred extension, that is finally returned when the solver cannot find any (larger) solution.

\begin{algorithm}[!h]
\begin{algorithmic}
        \STATE $P = $ PB problem with constraints {\bf (1)}, {\bf (2)}, {\bf (3)} and {\bf (4)} \;
        \WHILE{$P.solve() \not= null$}
            \STATE $E \gets P.solve()$
            \STATE $P.add\_constraint(x_{1} + x_{2} + \dots + x_{n} = n)$,  with $E = \{a_1, a_2, \dots, a_n \}$
            \STATE $P.add\_constraint(x_{1} + x_{2} + \dots + x_{n} \geq 1)$,  with $\A \setminus E = \{a'_1, a'_2, \dots, a'_n \}$
        \ENDWHILE
        \RETURN $E$
\end{algorithmic}
\caption{Compute a strong preferred extension}\label{computePrefAlgorithm}
\end{algorithm}


\section{Experimental Evaluation}\label{section:experiments}

For estimating the scalability of our method based on pseudo-Boolean constraints, we present now some results obtained from our experimental evaluation using two prominent PB solvers: Sat4j \cite{LeBerreP10} and RoundingSat \cite{ElffersN18}. While Sat4j is based on saturation, RoundingSat uses the division rule (see \cite{ElffersN18} for a discussion on both approaches). We focus here on the most relevant results; full results are presented in the appendix.

\paragraph{Benchmark Generation}
We generate benchmarks in a format adapted to StrAFs, inspired by ASPARTIX formalism \cite{DvorakGRWW20}. We consider two classes of randomly generated graphs. First, with the Erd\"os–R\'enyi model (ER) \cite{ErdosRenyi59}, given a set of arguments $\A$, and $p \in [0,1]$, we generate a graph such that for each $(a,b) \in \A \times \A$, $a$ attacks $b$ with a probability $p$. We consider two values for the probability, namely $p \in \{ 0.1, 0.5\}$. Then, with the Barab\'asi–Albert (BA) model \cite{BarabasiAlbert01}, a graph of $n$ nodes is grown by attaching new nodes with $m$ edges that are preferentially attached to existing nodes with a high degree. These types of graphs have been frequently used for studying computational aspects of formal argumentation, in particular during the ICCMA competitions \cite{GagglLMW20}.
%
The choice of a generation model provides the arguments $\A$ and attacks $\R$. We attach a random strength $\s(a) \in \{1,\dots,20\}$ to each  $a \in \A$.
For each generation model, we build $20$ StrAFs for each $|\A| \in \{5, 10, 15, \dots, 60\}$.
Parameters ($p \in \{0.1, 0.5\}$ for ER, $m = 1$ for BA) are chosen to avoid graphs with a high density of attacks, that would prevent the existence of meaningful extensions ({\em e.g.} non-empty ones).
Larger StrAFs (with $|\A| \in \{5,10,\dots,250\}$) have been generated with the same parameters ($p \in \{0.1, 0.5\}$ for ER, $m = 1$ for BA) for studying the problem of providing one extension.

\paragraph{Experimental Setting} The experiments were run on a Windows computer (using Windows Subsystem for Linux), with an Intel Core i5-6600K 3.50GHz CPU and 16GB of RAM. The timeout is set to $600$ seconds (same as the timeout at ICCMA \cite{LagniezLMR20}). 

\paragraph{Results} We are interested in the semantics $\sigma_X$, with $\sigma \in \{\pr,\stb,\co\}$ and $X \in \{S,W\}$.
The encodings for 
$\stb_X$ ($X \in \{S,W\}$)
are those proposed in \cite{RossitMDM21}, while the encoding for the other semantics are those described in Section~\ref{section:pb-encoding}. For each generated $\straf$, and each of these semantics $\sigma_X$, the two tasks we are interested in consist in enumerating all extensions and finding one extension.
We first focus on the runtime for enumerating $\sigma_X$ extensions, which provides an upper bound of the runtime for solving other classical reasoning tasks. 
 To do so, we use a Python script that converts a StrAF into a set of PB constraints. The set of extensions is then obtained in a classical iterative way: once an extension is returned by the PB solver, we add a new constraint that forbids this extension, and we call again the solver on this updated set of PB constraints. This process is repeated until the set of constraints becomes unsatisfiable, which means that all the extensions have been obtained.
Concerning the preferred extensions, this iterative approach is combined with Algorithm~\ref{computePrefAlgorithm}. In order to measure the performance of our approach, and since there is no other computational approach for StrAF semantics yet, we also implemented a so-called \textit{naive} algorithm that enumerates all sets of arguments and then checks, for each of them, if it is a $\sigma_X$ extension.
Figure~\ref{fig:enumeration-runtime} presents the average runtimes w.r.t. instance sizes ({\em i.e.} $|\A|$) for various semantics and StrAF families as described before. 
As a first result, we observe in Figure~\ref{fig:weak-complete-semantics-10p} that runtime for enumerating extensions (with the PB approach) is reasonable (\textit{i.e.} less than a minute) for most of the cases considered in our study, when the PB approach is used, while the naive approach reaches the timeout for most of the large instances (in particular, all the instances with $|\A| \geq 45$). The average runtimes are higher in only two situations: the enumeration of strong preferred and strong complete extensions, with the BA graphs.
However, even in such situations where the enumeration is harder ({\em e.g.} for $\pr_S$-extensions on BA graphs, as depicted on Figure~\ref{fig:strong-preferred-semantics-bara-enum}), the PB solvers clearly outperform the naive algorithm, which reaches the timeout in every instance when $|\A| \geq 30$, while the PB approach can enumerate extensions for larger graphs.
%

\begin{figure}[htb]
    \centering
    \subfloat[$\co_W$ on ER graphs ($p = 0.1$)\label{fig:weak-complete-semantics-10p}]{
    \includegraphics[scale=0.35]{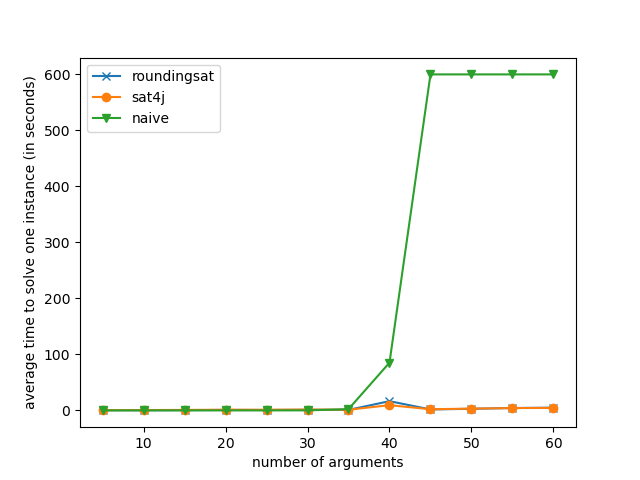} 
    }
    \subfloat[$\pr_S$ on BA graphs\label{fig:strong-preferred-semantics-bara-enum}]{
    \includegraphics[scale=0.35]{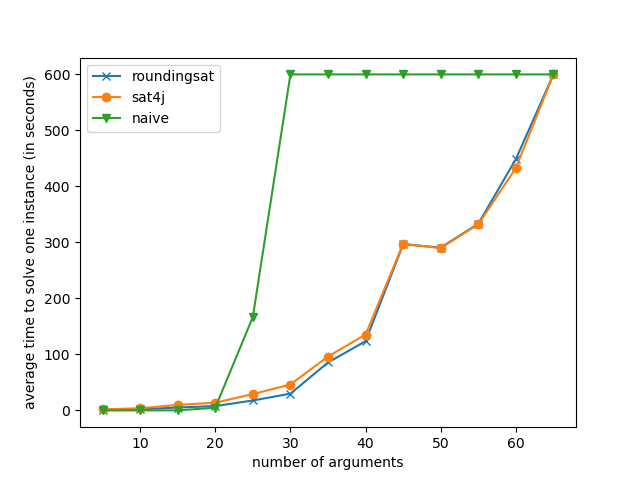} 
    }
    \caption{Enumeration runtime}
    \label{fig:enumeration-runtime}
\end{figure}


We also study the classical problem of providing one extension, for StrAFs of larger sizes (recall that here $|\A| \in \{5,10,\dots, 250\}$).
Figure~\ref{fig:weak-preferred-semantics-bara-one} shows that the PB solvers (in particular, Sat4j) provide one extension for these large graphs under two minutes, even for the preferred semantics (which is the hardest one, in our study, from the computational point of view).
Concerning the respective performances of the two PB solvers, Figure~\ref{fig:weak-preferred-semantics-bara-one} shows that RoundingSat processes faster for fast-to-compute instances ({\em i.e.} the smallest ones), while Sat4j outperforms it for instances of larger size. While we do not have explanations for this phenomenon, a plausible assumption is that it is related to the difference of the underlying algorithms (saturation for Sat4j and division rule for RoundingSat). Similar things have been observed for SAT solvers used in the case of standard AFs \cite{GningM20}.


\begin{figure}[htb]
\centering
\includegraphics[scale=0.35]{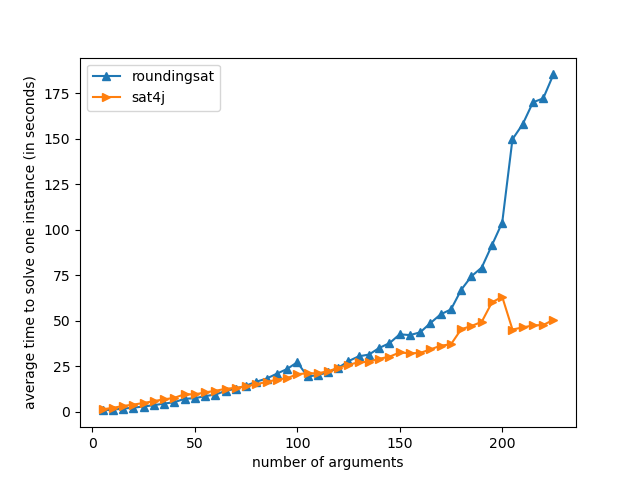} 
\caption{Finding one extension runtime under $\pr_s$ on BA graphs
\label{fig:weak-preferred-semantics-bara-one}}
\end{figure}

As a general conclusion on our experimental analysis, we observe that the PB approach for reasoning with StrAFs generally scales up well, for both problems of enumerating extensions and providing one extension.

\section{Conclusion}\label{section:conclusion}
Strength-based Argumentation Frameworks (StrAFs) have originally been proposed in~\cite{RossitMDM21}.
Contrary to this work, in this paper we focused on admissibility-based semantics. We showed that the weak admissibility-based semantics defined in the original work satisfy some expected properties, namely Dung's Fundamental Lemma. However, the definition for strong admissibility proposed in \cite{RossitMDM21} does not yield semantics that behave as expected. This has conducted us to revisit the definition of strong admissibility, and this allowed us to introduce strong complete and preferred semantics. We have also enhanced the StrAFs literature by studying the computational complexity of classical reasoning problems for these semantics, and we have shown that it is the same as for the corresponding tasks in Dung's framework, in spite of the increase of expressivity. Then we have proposed a method based on pseudo-Boolean constraints for computing the extensions of a StrAF under the various semantics defined in this paper, and we have empirically evaluated the scalability of this approach for the new semantics defined in this paper, as well as the (weak and strong) stable semantics from \cite{RossitMDM21}. 

As future work we have identified several promising research tracks, including the study of (weak and strong) grounded semantics, and tight complexity results for the skeptical reasoning under the (weak and strong) complete semantics.
We are also interested in an analysis of the relation between StrAFs and other frameworks, in particular the comparison of the signatures of StrAFs semantics and SETAFs semantics \cite{NielsenP2006,DvorakFW19,FlourisB19}. Finally, we want to study argument strength and accrual in a context of structured argumentation.

\appendix

\section{Experimental Results}
This section describes additional experimental results, not presented in the main part of the paper because of space constraints.

\begin{figure}[htb]
    \centering
    \subfloat[ER graphs with p = 0.1\label{fig:strong-complete-semantics-10p}]{
    \includegraphics[scale=0.35]{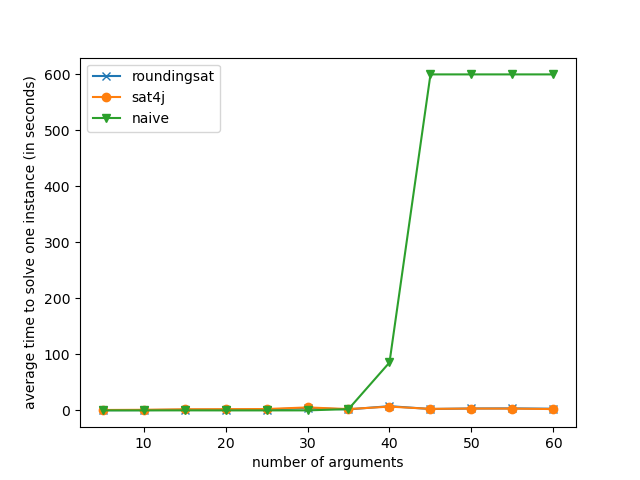} 
    }
    \subfloat[ER graphs with p = 0.5\label{fig:strong-complete-semantics-50p}]{
    \includegraphics[scale=0.35]{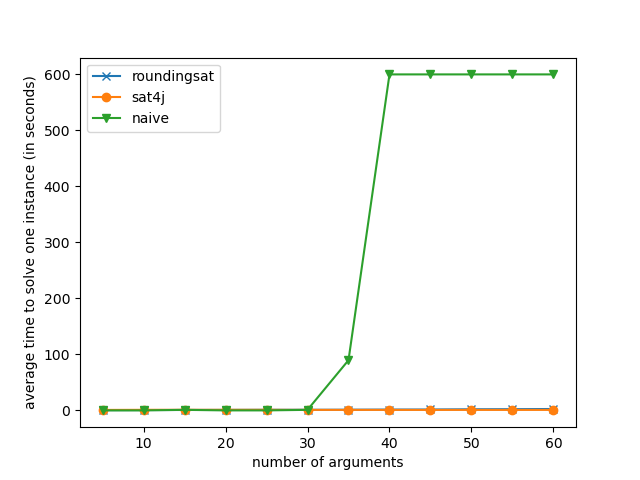}
    }
    
    \subfloat[BA graphs \label{fig:strong-complete-semantics-bara}]{
    \includegraphics[scale=0.35]{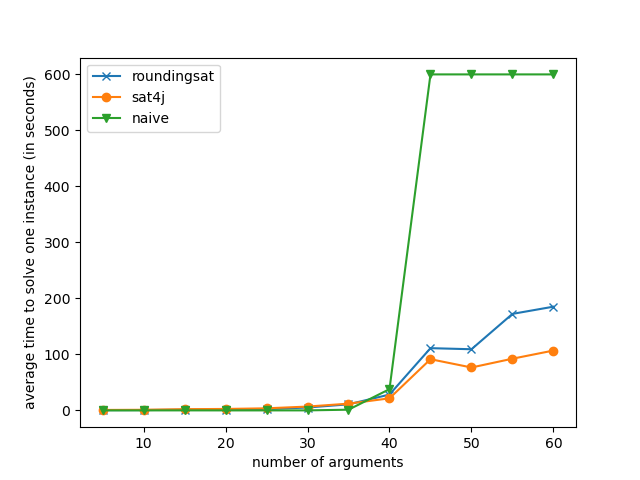} 
    }
    \caption{Enumeration runtime under $\co_{S}$ for various types of graphs}
\end{figure}




\begin{figure}[htb]
    \centering
    \subfloat[ER graphs with p = 0.1\label{fig:weak-complete-semantics-10p}]{
    \includegraphics[scale=0.35]{Images/wcom10pcomp.png} 
    }
    \subfloat[ER graphs with p = 0.5\label{fig:weak-complete-semantics-50p}]{
    \includegraphics[scale=0.35]{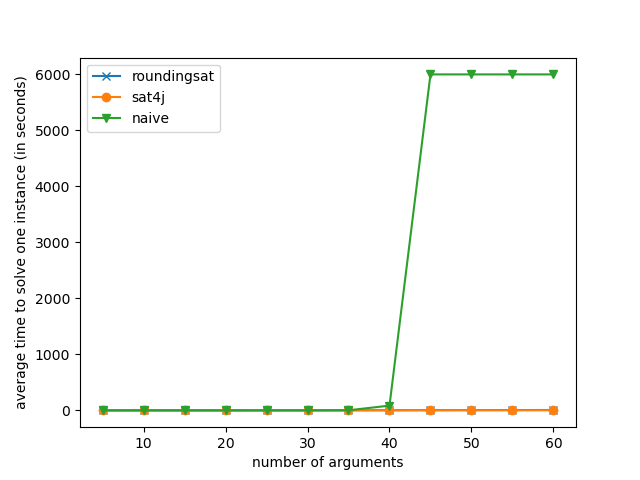}
    }
    
    \subfloat[BA graphs \label{fig:weak-complete-semantics-bara}]{
    \includegraphics[scale=0.35]{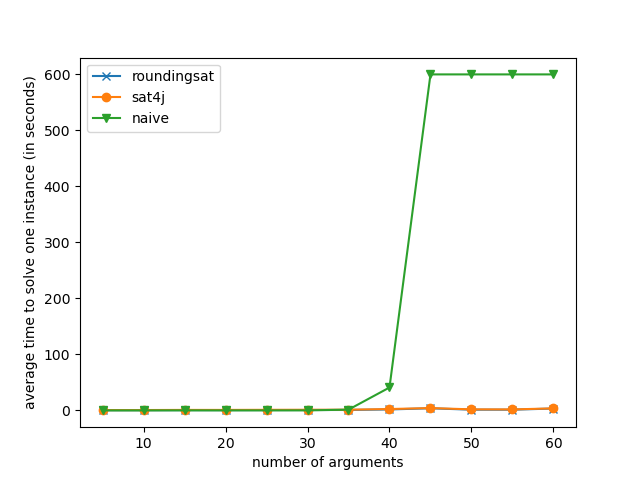} 
    }
    \caption{Enumeration runtime under $\co_{W}$ for various types of graphs}
\end{figure}




\begin{figure}[htb]
    \centering
    \subfloat[ER graphs with p = 0.1\label{fig:strong-preferred-semantics-10p}]{
    \includegraphics[scale=0.35]{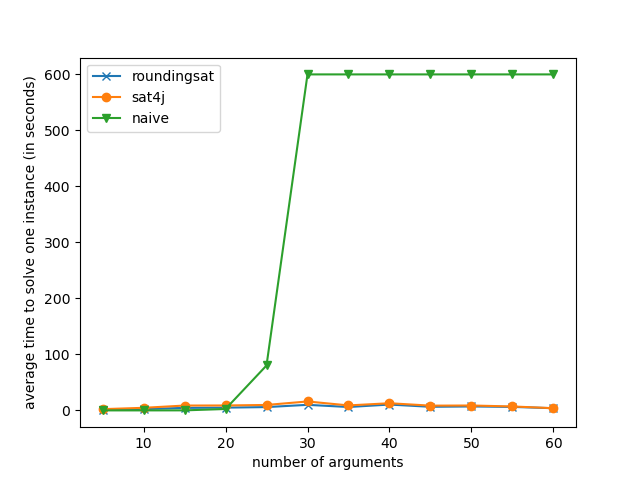} 
    }
    \subfloat[ER graphs with p = 0.5\label{fig:strong-preferred-semantics-50p}]{
    \includegraphics[scale=0.35]{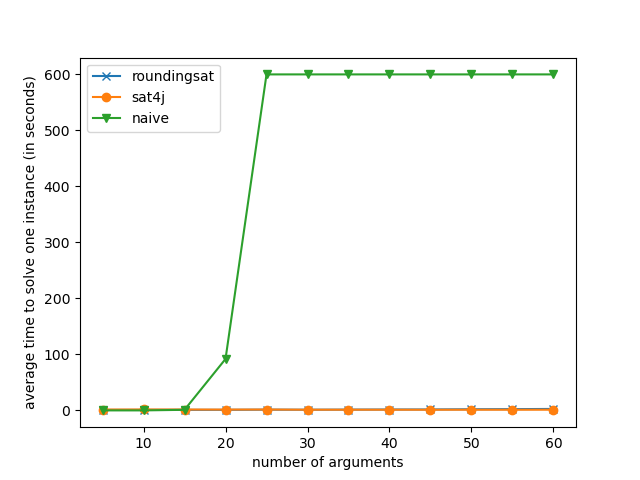}
    }
    
    \subfloat[BA graphs \label{fig:strong-preferred-semantics-bara}]{
    \includegraphics[scale=0.35]{Images/sprefbaracomp.png} 
    }
    \caption{Enumeration runtime under $\pr_{S}$ for various types of graphs}
\end{figure}




\begin{figure}[htb]
    \centering
    \subfloat[ER graphs with p = 0.1\label{fig:weak-preferred-semantics-10p}]{
    \includegraphics[scale=0.35]{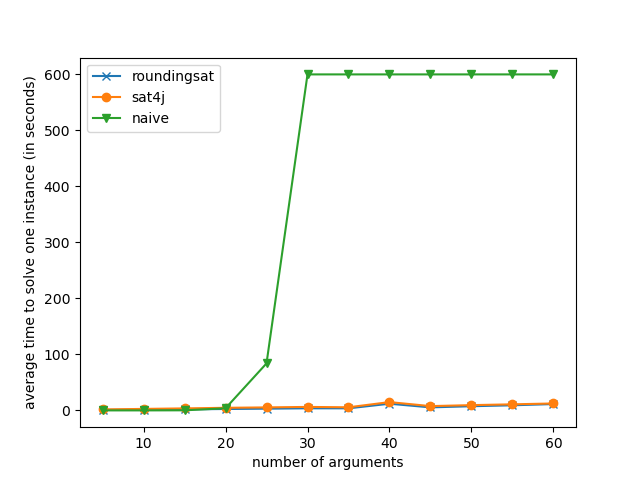} 
    }
    \subfloat[ER graphs with p = 0.5\label{fig:weak-preferred-semantics-50p}]{
    \includegraphics[scale=0.35]{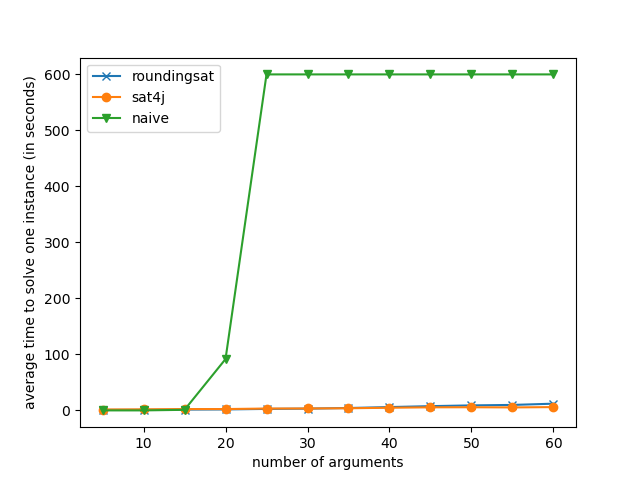}
    }
    
    \subfloat[BA graphs \label{fig:weak-preferred-semantics-bara}]{
    \includegraphics[scale=0.35]{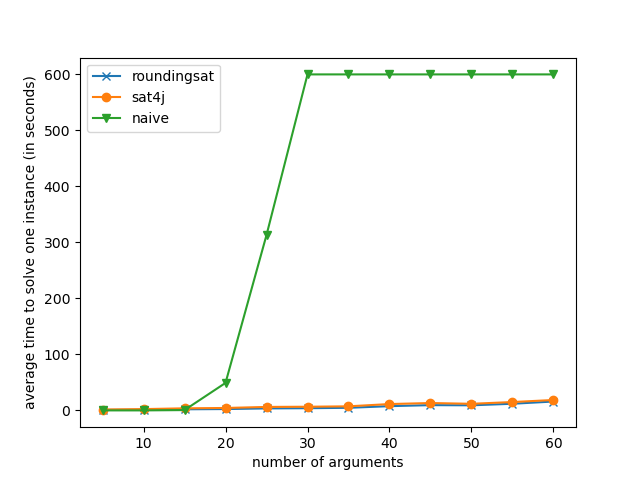} 
    }
    \caption{Enumeration runtime under $\pr_{W}$ for various types of graphs}
\end{figure}




\begin{figure}[htb]
    \centering
    \subfloat[ER graphs with p = 0.1\label{fig:strong-stable-semantics-10p}]{
    \includegraphics[scale=0.35]{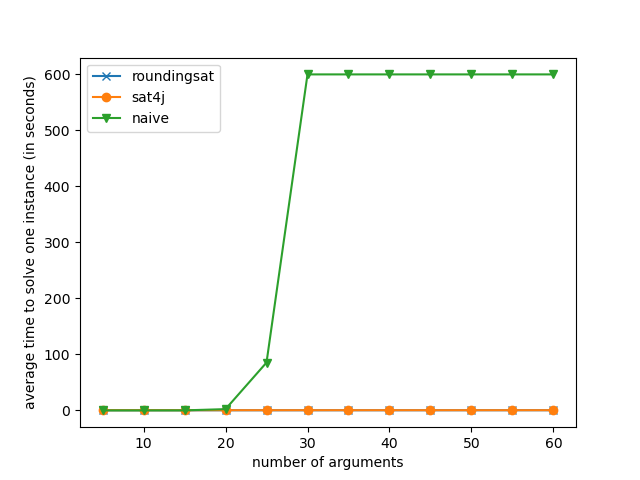} 
    }
    \subfloat[ER graphs with p = 0.5\label{fig:strong-stable-semantics-50p}]{
    \includegraphics[scale=0.35]{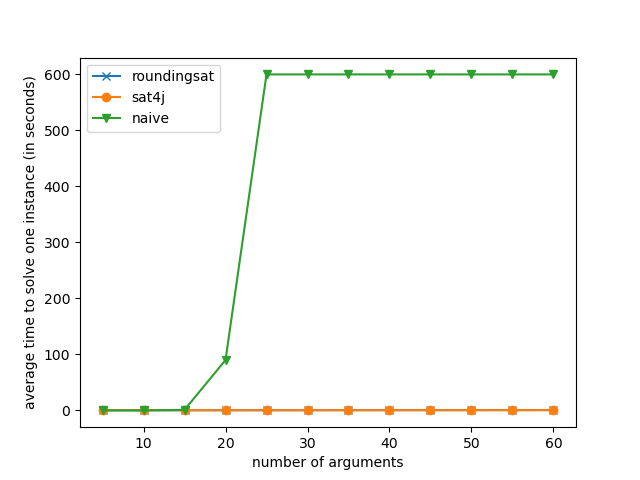}
    }
    
    \subfloat[BA graphs \label{fig:strong-stable-semantics-bara}]{
    \includegraphics[scale=0.35]{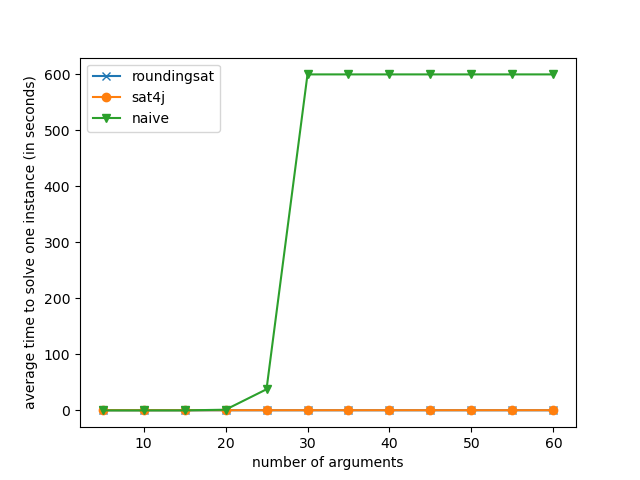} 
    }
    \caption{Enumeration runtime under $\stb_{S}$ for various types of graphs}
\end{figure}




\begin{figure}[htb]
    \centering
    \subfloat[ER graphs with p = 0.1\label{fig:weak-stable-semantics-10p}]{
    \includegraphics[scale=0.35]{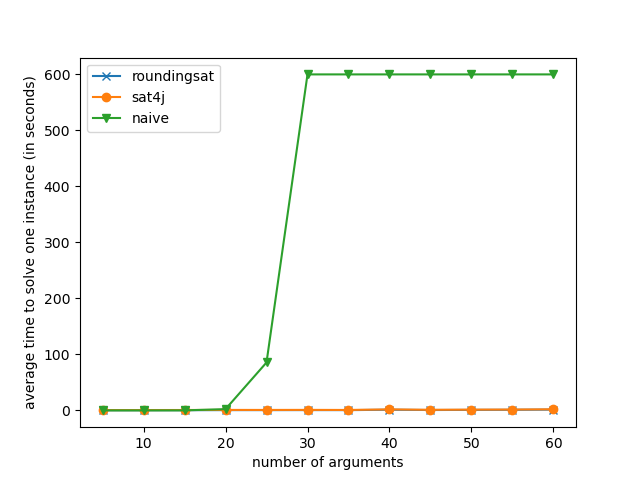} 
    }
    \subfloat[ER graphs with p = 0.5\label{fig:weak-stable-semantics-50p}]{
    \includegraphics[scale=0.35]{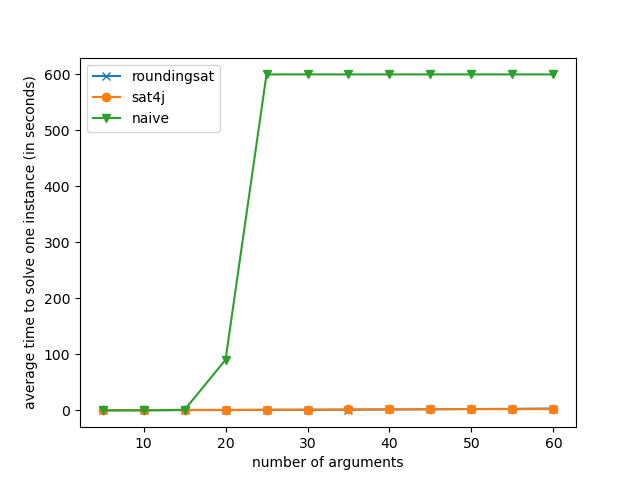}
    }
    
    \subfloat[BA graphs \label{fig:weak-stable-semantics-bara}]{
    \includegraphics[scale=0.35]{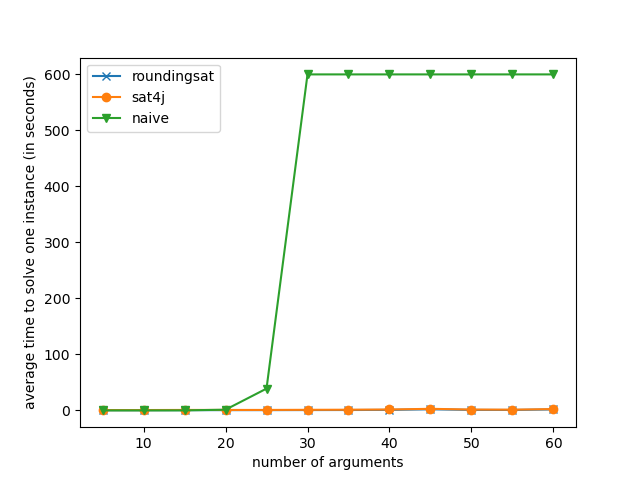} 
    }
    \caption{Enumeration runtime under $\stb_{W}$ for various types of graphs}
\end{figure}




%
\bibliographystyle{unsrt}  
\bibliography{references} 
\end{document}